\documentclass{article}

\usepackage{microtype}
\usepackage{graphicx}
\usepackage{subfigure}
\usepackage{booktabs} % for professional tables

\usepackage{hyperref}
\usepackage{float}
\usepackage{algorithm2e}

\SetKwInput{KwInput}{Input}                % Set the Input
\SetKwInput{KwOutput}{Output}

\usepackage[accepted]{icml2023}

\usepackage{amsmath}
\usepackage{amssymb}
\usepackage{mathtools}
\usepackage{amsthm}

\usepackage[capitalize,noabbrev]{cleveref}

\theoremstyle{plain}
\newtheorem{theorem}{Theorem}[section]

\newtheorem{lemma}[theorem]{Lemma}

\theoremstyle{definition}
\newtheorem{definition}[theorem]{Definition}

\theoremstyle{remark}
\newtheorem{remark}[theorem]{Remark}

\usepackage[textsize=tiny]{todonotes}

\newcommand{\tr}[1]{\mathsf{Tr}\left(#1\right)}
\newcommand{\norm}[1]{\|#1\|}
\newcommand{\R}{\mathbb{R}}
  \usepackage{amsmath}
  \usepackage{amsthm}
  \usepackage{mathrsfs}

  \newcommand{\eps}{\varepsilon}

\newcommand{\youngsuk}[1]{} 
\newcommand{\hilaf}[1]{}
\newcommand{\bernie}[1]{}
\newcommand{\arun}[1]{}

  \theoremstyle{plain}

  \theoremstyle{definition}

  \theoremstyle{remark}

  \numberwithin{equation}{section}
  \numberwithin{thm}{section}

\icmltitlerunning{Testing Causality for High Dimensional Data}

\begin{document}

\twocolumn[
\icmltitle{Testing Causality for High Dimensional Data}



\icmlsetsymbol{equal}{*}

\begin{icmlauthorlist}
\icmlauthor{Arun Jambulapati}{equal,yyy}
\icmlauthor{Hilaf Hasson}{comp}
\icmlauthor{Youngsuk Park}{comp}
\icmlauthor{Yuyang Wang}{comp}
% \icmlauthor{Firstname1 Lastname1}{equal,yyy}
% \icmlauthor{Firstname2 Lastname2}{equal,yyy,comp}
% \icmlauthor{Firstname3 Lastname3}{comp}
\end{icmlauthorlist}

\icmlaffiliation{yyy}{Department of Computer Science, University of Washington, Seattle}
\icmlaffiliation{comp}{AWS AI Labs, Santa Clara}

\icmlcorrespondingauthor{Yuyang Wang}{yuyawang@amazon.com}

\icmlkeywords{Machine Learning, ICML}

\vskip 0.3in
]



\printAffiliationsAndNotice{$^*$Work done during internship at AWS AI Labs.} 
\begin{abstract}
    Determining causal relationship between high dimensional observations are among the most important tasks in scientific discoveries. In this paper, we revisited the \emph{linear trace method}, a technique proposed in~\citep{janzing2009telling,zscheischler2011testing} to infer the causal direction between two random variables of high dimensions.  We strengthen the existing results significantly by providing an improved tail analysis in addition to extending the results to nonlinear trace functionals with sharper confidence bounds under certain distributional assumptions. We obtain our results by interpreting the trace estimator in the causal regime as a function over random orthogonal matrices, where the concentration of Lipschitz functions over such space could be applied. We additionally propose a novel ridge-regularized variant of the estimator in \cite{zscheischler2011testing}, and give provable bounds relating the ridge-estimated terms to their ground-truth counterparts. We support our theoretical results with encouraging experiments on synthetic datasets, more prominently, under high-dimension low sample size regime. 
\end{abstract}

\section{Introduction}
%\bw{needs an entry paragraph to set the stage.}

Over the last decades, Machine Learning (ML) has found tremendous successes in application areas such as image recognition, natural language processing and computer vision. However, at its core, the strength of the mainstream ML models such as deep neural networks and tree-based approaches predominantly lies with \emph{making predictions}, whereas the fundamental component of human intelligence is about \emph{causal reasoning.} As such, considered where the hard problems reside~\citep{scholkopf2019causality}, the topic of causality~\citep{pearl2009causality} has attracted significant interests in the ML community in recent years in the search of algorithms that could ``climb the ladder of causality.'' This paper falls into such a quest, with a goal to quantify the causal relationship between two high-dimensional variables. 

% \blfootnote{
% \noindent\textsuperscript{\textbf{*}}Work done as an intern at AWS AI Labs.\textsuperscript{$\dagger$}Corresponding author.}

Concretely, the problem of concern is to infer whether the linear relations between two high-dimensional variables $X$ taking values in $\R^n$ and $Y$ taking values in $\R^m$ are due to a causal influence from $X$ to $Y$ or from $Y$ to $X$. We adopt the intuition from~\cite{janzing2009telling} that uses the concept of \emph{asymmetries} of the joint distribution and assumes that the causal factorization 
\[
P(\text{effect}, \text{cause}) = P(\text{cause})P(\text{effect}|\text{cause})
\]
admits ``simpler'' terms than the non-causal ones. To be more precise, consider the problem of which of the following two models is more plausible as a causal model, 
\[
Y = AX + E\quad \text{and}\quad X = \tilde{A}Y + \tilde{E},
\]
where $E$ and $\tilde{E}$ are independent noise terms. The intuition is that given that $X$ is the \emph{cause}, then $\Sigma_{XX}$, the covariance matrix of $X$, is in some sense ``independent'' from the coefficient matrix $A$. To this aim, \cite{janzing2009telling} show that it boils down to quantifying how much the following relation is violated, 
\begin{equation}
\label{eqn:tr}
    \tau_m(A \Sigma_{XX} A^\top) \approx \tau_m(A A^\top)\tau_n(\Sigma_{XX}),
\end{equation}
$X \in \R^n$, $E, Y \in \R^m$, $A \in \R^{m \times n}$, and  $\tau_n: A \mapsto \text{tr}(A) / n$ is the normalized trace. Intuitively, if $A$ and $\Sigma_{XX}$ are chosen independently, (\ref{eqn:tr}) should roughly hold. The intuition here is that general positive semi-definite matrices with the same spectrum as $\Sigma_{XX}$ can be given a distribution as the pushforward of the Haar measure on the orthogonal group $O(n)$, 
%\youngsuk{better define $O(n)$ as it first appeared?, what is diff. between $SO(n)$ and $O(n)$?}
acting through conjugation: $U\mapsto U \Sigma_{XX}U^{\top}$. Under this distribution $\tau_m(AU \Sigma_{XX}U^\top A^\top)$ is centered around $\tau_m(AA^{\top})\tau_n(\Sigma_{XX})$, equivalently $\tau_m(AA^{\top})\tau_n(U \Sigma_{XX}U^\top)$, as we shall see in the proof of Theorem \ref{thm:ours}; and the further away it is from the center, the more ``unlikely'' the pair $A$ and $\Sigma_{XX}$.
%\youngsuk{can be shorten a bit, seems to be too long as an intuition? we don't need to metnion our theorem 2.3. here as we are describing basically janzing's work?. }
%\hilaf{It would be great to reference here where to look in the paper in order to get this intuition. Currently the best I can say is the proof of Theorem 2.2...} 
This leads to the development of the \emph{Trace Method} proposed in~\cite{janzing2009telling}. Define the Delta estimator as
\begin{equation}
\label{eqn:deltaEstimator}
\Delta_{X\rightarrow Y} := \log\frac{\tau_m( A\Sigma_{XX} A^\top)}{\tau_m(A A^\top)\tau_n(\Sigma_{XX})},
\end{equation}
%\hilaf{This used to say $\tau_m(AA^\top)$...}
the variable $X$ is determined as the ``cause'' if
\[
\Delta_{X\rightarrow Y} \geq \Delta_{Y\rightarrow X} + \xi
\]
for some tolerance $\xi\in\R^+.$ The trace method is an example of information-geometric causal inference: we refer the reader to Section 4.1 of~\cite{peters2017elements} for more detail. 
% \footnote{
Acute readers might notice that in \cite{janzing2009telling}, the $\Delta$ estimator is defined with an absolute value, and the sign of the inequality is flipped. We note that in the noiseless regime and in the non-causal setting the $\Delta$ estimator as defined here is negative: thus the formula given here matches their notation. Not applying the absolute value here allows us to generalize this $\Delta$ parameter to the noisy regime.
%\youngsuk{I heard Arun explained this last time, but to be honest, could not follow. Anyone @bernie, @hilaf who can guide me?}

The key theoretical contribution from~\cite{janzing2009telling} is a concentration result that shows that, with $U$ being Haar-distributed (``uniformly distributed'') on the orthogonal group $O(n),$ the deviation
\[
\tau_m({A U C U^\top A^\top}) - \tau_n(C)\tau_m(A A^\top)
\]
%\hilaf{This used to say: $\tau_m({A^\top U C U^\top A}) - \tau_n(C)\tau_m(A^\top A)$.}
is small with high probability (w.h.p.). However, we note that the original proof in~\cite{janzing2009telling} was incorrect (see Remark 1), and as our first main contribution (\Cref{thm:ours}), we repair the proof and strengthen the result by a factor of $O(1/\sqrt{m}).$ We obtain this result by interpreting the trace estimator in the causal regime as a function over random orthogonal matrices and by applying known results on concentration of Lipschitz functions over this space.  

Extending our proof technique, we additionally provide high-probability bounds for a family of $\Delta$ estimators based on Schatten norms (\Cref{thm:ours2}). More specifically, we define the estimator 
\[
\Delta^{(p)}_{X \rightarrow Y} :=   \log \frac{\tau_m \left( \left( A  \Sigma_{XX} A^\top \right)^p \right) }{\mathbb{E}_{U \sim O(n)} \tau_m  \left( \left( A  U \Sigma_{XX} U^\top A^\top \right)^p \right)} 
\]
and prove that $\Delta^{(p)}_{X \rightarrow Y} \approx 0$ with high probability under the same generative assumptions as \cite{janzing2009telling}. This generalization to $p^{th}$ moments is inspired by recent work \citep{Jambulapati0T20} employing such potentials as robust proxies for the maximum eigenvalue: we believe that the resulting causality estimators may find application in settings where the eigenvalues of the covariance matrices are significantly non-uniform.

In practice, since the true covariance matrices $\Sigma_{XX}, \Sigma_{YY}$ and the structure matrices $A, \tilde{A}$ are unknown, one needs to replace them with their empirical estimates. However, a major drawback that hinders the practicality of the Trace Method is the unfavorable sample complexity that stems from the estimation of the aforementioned quantities. To alleviate this issue, in a follow up work, \cite{zscheischler2011testing} proposed a finite sample version of the Delta estimator, which replaces the true covariances with rank-corrected empirical estimates. More precisely, given a set of observations $(X_1, Y_1), (X_2, Y_2), \dots$ we construct the \emph{empirical delta estimator},
\begin{equation}
\label{eqn:naive_delta_est}
\hat{\Delta}_{X\rightarrow Y} := \log\frac{\tau_m(\widehat{A}C_{XX}\widehat{A}^\top)}{n/r\cdot\tau_m(\widehat{A}\widehat{A}^\top)\tau_n(C_{XX})},
\end{equation}
%\hilaf{The primes here are $^\top$s, right? Shall we make the change or am I missing something?}

where $C_{XX}, C_{XY}$ are the empirical estimator of $\Sigma_{XX}, \Sigma_{XY}$, $r$ is the rank of $C_{XX}$ and $\hat{A} = C_{YX} C_{XX}^\dagger$: here $\dagger$ denotes pseudo-inverse. Using \emph{free probability} (non-commutative probability), the authors show that $\hat{\Delta}_{X\rightarrow Y} \rightarrow 0$ as the number of samples tends to infinity, if $X$ causes $Y$ and the covariance matrix of $X$ is generated by a rotationally invariant ensemble.\footnote{Here, a rotationally invariant ensemble is a distribution which is unaffected by the application of any fixed rotation matrix $U$.}   This yields a test statistic that works without good estimates of $A$ and $C_{XX}.$ However, a closed-form of the distribution of $\hat{\Delta}_{X_n\rightarrow Y_n}$ under the null hypothesis that the model is causal is not attainable, and the authors resort to a heuristic algorithm to assess the significance level. Additionally, the practical performance of the finite-sample estimator is unsatisfactory, especially when the number of samples is smaller than the ambient dimension. This raises a question: is there a better test with provable guarantees? Our second main contribution answers this question positively. Motivated by ridge regression, we propose a novel variant of the estimator in \cite{zscheischler2011testing}, leveraging the regularized estimate of the structural matrix $\hat{A}_\lambda$. We prove asymptotic bounds on the numerator and denominator of the proposed estimator,
\[
\widetilde{\Delta}_{X \rightarrow Y}^{\lambda} :=  \log \frac{\tau_m ( \hat{A}_{\lambda} C_{XX} \hat{A}_{\lambda}^\top )} {\tau_m(\hat{A}_{\lambda} \hat{A}_{\lambda}^\top) \tau_n (C_{XX}) } 
\]
in the presence of noise, a first result of this kind. Our techniques unfortunately fall just short of obtaining complete bias bounds on the full estimator and also cannot give finite-sample guarantees: we leave this as an intriguing direction for future work. 

%give provable bounds on its performance in the finite-sample regime. \bw{@Arun, expand with the precise statement.}

The rest of the paper is organized as follows. In \Cref{sec:tm} we recap the trace method with its main theoretical argument. We provide improved concentration results with a complete proof, and further extend the results to allow for refined tail analysis of this type of problem. \Cref{sec:rr} is dedicated to the ridge-regularized trace estimator and its theoretical properties. We conclude in \Cref{sec:experiments} with experiments which justify our findings on synthetic datasets. % and provide some closing thoughts.

%Under the assumption of $\Sigma_X$ from a rotational invariant ensemble, i.e., the (element) density is invariant under conjugation of a Haar distributed matrix, the authors~\cite{zscheischler2011testing} proved that $\hat{\Delta}_{X\rightarrow Y} \rightarrow 0$ as $n$ approaches infinity. 
% \youngsuk{if X acually caused Y.}
% This key result paves the way to identify the causal relation and even with hidden con-founders, see~\cite{zscheischler2011testing} for more detail.

\section{Trace Method for Causality}
\label{sec:tm}
\subsection{Improved Concentration Results}
In this section, we prove \Cref{thm:ours} and \Cref{thm:ours2}, our main concentration bounds for a general family of Delta estimators extending the work of \cite{janzing2009telling, zscheischler2011testing}. 
To this end, we first formally specify our distributional assumptions:
\begin{definition}[Linear causal model]
\label{def:model}
Let $\Lambda_{XX}$ be a nonnegative diagonal matrix, and let $U$ be an orthogonal matrix sampled from the Haar measure on $O(n)$. Let $A$ be a (potentially random) matrix generated independently from $U$. We define a Linear causal model with cause $X$ as a process which, after fixing $U$ and $A$ as sampled above, samples $X \sim N(\mu_x, \Sigma_{XX} = U \Lambda_{XX} U^\top)$ forms $Y = AX + E$, where $E$ is a noise variable independent from $X$. We say a set of samples $(X_i, Y_i)$ is from a Linear causal model with causal variable $X$ if each $(X_i, Y_i)$ pair is independently sampled as the above. 
\end{definition}

Under these assumptions, \cite{janzing2009telling} observes that the following two quantities should be close if $X$ ``causes'' $Y,$ (recall that $C_{XX}$ is the emprical estimate of the covairance matrix of $X$)
\[
\tau_m(AC_{XX}A^\top ), \quad \text{and}, \quad \tau_n(C_{XX})\tau_m(A A^\top ),
\]
%\youngsuk{better remind $C_{XX}?$}
%\hilaf{This would require a lot of changes everywhere, but wouldn't it be a lot more readable if it were $\tau_m(A^\top A)$ instead?}
where we define the trace functional
$
\tau_n(M) = \frac{1}{n} \tr{M}
$
for $M \in \R^{n \times n}$. 
% \youngsuk{The statement above is kind of repetitive to one in the introduction. May shorten it or at least say 'as described in ...'}
The motivating theorem behind the authors' \emph{Trace Method} is the following concentration bound on the quantities of interest in the causal direction. 
%\hilaf{$A^\top=A^T$}
\begin{theorem}[Theorem 1, \cite{janzing2009telling}]
\label{thm:janzing}
Let $C$ be a symmetric, positive definite $n \times n$-matrix and $A$ an arbitrary $m \times n$-matrix. Let $U$ be drawn from the Haar measure over $O(n)$. \
%\youngsuk{define $O(n)$ somewhere} 
For any $\eps \geq 0$, we have
\[
\begin{split}
    &\left|\tau_m({A U C U^\top A^\top}) - \tau_n(C)\tau_m(A A^\top) \right| \leq 2\eps \norm{C} \norm{A A^\top}
\end{split}
\]
with probability at least $1 - \exp \left( -\kappa_1 (n-1) \eps^2 \right)$ for some constant $\kappa_1$. Equivalently, with confidence $1-\delta$, 
% \youngsuk{with confidence <-> w.p. at least $1-\delta$?}
we have
\[
\begin{split}
    &\left|\tau_m({A U C U^\top A^\top}) - \tau_n(C)\tau_m(A A^\top) \right| \\
    &\qquad\qquad \leq \left[2\sqrt{\frac{\log(2/\delta)}{\kappa_1(n-1)}}\right] \norm{C} \norm{A A^\top}.
\end{split}
\]
\end{theorem}
%\youngsuk{Does this theorem rely on $SO(n)$ during the proof? I just want to understand, all the argument changing from $SO(n)$ to $O(n)$ we will have in our theorem 2.3.  also exists in the proof of this theorem or is only necessary to our case.}

Our first main result is a strengthening of the above theorem, in the following sense.
\begin{theorem}
\label{thm:ours}
Under the same assumptions as in Theorem~\ref{thm:janzing}, denote $\beta_i, \gamma_i$ as the eigenvalues of matrices $A^\top A$ and $C$ respectively, we have 
\[
\left| \tau_m({A U C U^\top A^\top}) - \tau_n(C)\tau_m(A A^\top) \right|
\leq \frac{\eps}{m} \sqrt{\sum_{i=1}^n \beta_i^2\gamma_i^2 } 
\]
with probability at least $1 - 2\exp(-\kappa_2  n \eps^2)$ for some universal constant $\kappa_2$. Equivalently, with confidence $1-\delta$, we have
\[
\begin{split}
&\left| \tau_m({A U C U^\top A^\top}) - \tau_n(C)\tau_m(A A^\top) \right|\\
&\qquad\qquad\leq \left[\sqrt{\frac{2\log(2/\delta)}{\kappa_2n}}\right]\frac{1}{m} \sqrt{\sum_{i=1}^n \beta_i^2\gamma_i^2 }.
\end{split}
\]
\end{theorem}
%\hilaf{write same as Thm1 or vice versa; what's the deal with $c$?}
This bound is stronger by at least a factor of $O(1/\sqrt{m})$ than the one proved in \cite{janzing2009telling}. To see this, observe
\[
\frac{1}{m} \sqrt{ \sum_{i=1}^n \beta_i^2 \gamma_i^2 } \leq \frac{1}{m} \sqrt{ \mathsf{rank}(A A^\top) \beta_1^2 \gamma_1^2 } \leq \frac{\norm{A A^\top} \norm{C}}{\sqrt{m}}.
\]
%\youngsuk{what is $K$?}\hilaf{Be more explicit about the logic here...}

We shall prove the theorem in the rest of the section. Our proof proceeds via concentration inequalities for Lipschitz functions over $SO(n)$: this extends L\'evy's concentration of measure lemma as used in \cite{janzing2009telling}. 

\begin{lemma}[Theorem 5.2.7, \cite{HDP}]
\label{lemma:SO_conc}
Let $SO(n)$ be the group of orthogonal matrices, and let $\norm{\cdot}_F$ be the matrix Frobenius norm. For a function $f : SO(n) \rightarrow \R$, let $L_f$ be such that 
\[
\left| f(U) - f(V) \right| \leq L_f \norm{U - V}_F
\]
for any $U,V \in SO(n)$. Then if $U$ is sampled from the Haar measure over $SO(n)$, 
\[
\norm{f(U) - \mathbb{E} f(V)}_{\psi_2} \leq \frac{\alpha L_f}{\sqrt{n}}
\]
for some universal constant $\alpha \geq 0$. Here, the $\psi_2$-Orlicz norm $\norm{Z}_{\psi_2}$ (also known as the sub-Gaussian norm) is the smallest value $s$ such that $\mathbb{E} \left[e^{(Z/s)^2} \right] \leq 2$.
\end{lemma} 

We remark that this bound on the $\psi_2$-Orlicz norm is sufficient to establish subgaussian concentration inequalities via the following standard lemma:
\begin{lemma}[(2.14) of \cite{HDP}]
\label{lem:concentration}
Let $Z$ be a zero-mean random variable. Then
\[
\Pr\left( |Z| \geq t \right) \leq 2 \exp\left( - \frac{ t^2}{6 \norm{Z}_{\psi_2}^2} \right).
\]
\end{lemma}

To prove the tightest possible bounds, we require the von Neumann trace inequality:

\begin{lemma}
\label{lem:von-neumann}
Let $A, B$ be symmetric matrices. Let $\alpha_1 \leq \dots\leq \alpha_n$ and $\beta_1 \leq  \dots \leq \beta_n$ be the eigenvalues of $A$ and $B$ respectively. Then 
\[
\left| \tr{AB} \right| \leq \sum_{i=1}^n \alpha_i \beta_i.
\]
\end{lemma}

We will finally require a simple technical claim about convex combinations of orthogonal matrices.
%\hilaf{In the following lemma we are again using $X$ and $Y$ in roles that are entirely different than the roles they appear in in the setup. I am replacing all $X$ and $Y$ with $U$ and $V$, as was done in some but not all of the places in the lemma and proof.}
\begin{lemma}
\label{lem:conv_orthogonal}
Let $U,V \in \R^{n \times n}$ be orthogonal matrices. For $t \in [0,1]$, define $Z_t = t U + (1-t) V$. Then for any positive semidefinite matrix $A$ and $i \in [n]$, we have
\[
\lambda_i( Z_t A Z_t^\top) \leq \lambda_i(A)
\]
where $\lambda_i(A)$ denotes the $i^{th}$ largest eigenvalue of $A$.

\end{lemma}

\begin{proof}
We observe that since $A$ is positive semidefinite, there exists a symmetric matrix $B$ such that $A = B^2$. Thus, we have
\[
\lambda_i( Z_t A Z_t^\top ) = \sigma_i (Z_t B)^2
\]
where $\sigma_i$ denotes the $i^{th}$ largest singular value of the input matrix. By the min-max characterization of singular values, we have 
\[
\sigma_i( Z_t B) = \max_{\substack{W \\ \mathsf{dim}(W) = i}} \min_{\substack{v \in W \\ \norm{v} = 1}} \norm{Z_t B v}. 
\]
Now, triangle inequality of the Euclidean norm yields 
\begin{align*}
\norm{Z_t B v} &\leq t \norm{U B v} + (1-t) \norm{V B v}\\
&= t \norm{Bv} + (1-t) \norm{Bv} = \norm{Bv}
\end{align*}
where the equality in the second line follows from $U$ and $V$ be orthogonal matrices. Applying this fact yields $\sigma_i(Z_t B) \leq \sigma_i(B)$, and therefore
\[
\lambda_i(Z_t A Z_t^\top) = \sigma_i (Z_t B)^2 \leq \sigma_i(B)^2  = \lambda_i(A). 
\]

\end{proof}

\begin{proof}[Proof of Theorem~\ref{thm:ours}]
We first show the result holds if $U$ is drawn uniformly from $SO(n)$: we extend this result to the general setting of $O(n)$ at the end. Define $B = A^\top A$, and let $f(U) = \tr{B U C U^\top}$, for $U \in SO(n)$. We begin by showing $\mathbb{E} \left[ f(U)\right] = \frac{1}{n} \tr{B} \tr{C}$. We observe 
\begin{align*}
\mathbb{E}[f(U)] &= \mathbb{E}[ \tr{B U C U^\top} ] = \mathbb{E}[ \left\langle B, U C U^\top \right\rangle] \\
&=\left\langle B,  \mathbb{E}[ U C U^\top ] \right\rangle = \left\langle B, \frac{\tr{C}}{n} I \right\rangle\\
&= \frac{1}{n} \tr{B} \tr{C}.
\end{align*}
Here, we used the linearity of trace and the fact that $\mathbb{E}[ U C U^\top ] = \frac{1}{n} \tr{C} I$. %\hilaf{Am I correct that you're not using any knowledge about $B$, including that it is symmetric?}\arun{I implicitly used symmetry when writing the trace as an inner product, but this assumption can be removed.}

%\hilaf{I think it makes sense to mention somewhere here that the intuition for Equation (1) can be explained exactly by this formula for $U=I$.}

We now bound $L_f$. Observe that $f$ is differentiable over the space of matrices. For any $V \in SO(n)$, we define $Z_t = t U + (1-t) V$: we note
\begin{align*}
&f(U) - f(V) = \\ &\int_{t=0}^1 \nabla f(Z_t) \bullet (U-V) dt \leq \int_{t=0}^1 \norm{ \nabla f(Z_t) }_F \norm{U-V}_F dt \\
&\leq \left(\max_{t \in [0,1]} \norm{\nabla f(Z_t) }_F \right) \norm{U-V}_F.
\end{align*}
%\youngsuk{ $ \bullet$ is matrix inner product?}\hilaf{No, it's vector inner product.}\hilaf{Need to explain why the maximum of $f$ in the convex hull of $O(n)$ lies on the boundary.}
The first inequality follows from the Cauchy-Schwarz inequality $\tr{XY}^2 \leq \tr{X^\top X} \tr{Y^\top Y}$.  Thus it suffices to bound the maximum Frobenius norm of $\nabla f(Z_t)$ for $t \in [0,1]$. Straightforward calculation reveals $\nabla f(Z) = 2 C Z^\top B$ (where we used that $B$ and $C$ are symmetric). 

We observe that $C^2$ is a positive semidefinite matrix and that $Z_t$ is a convex combination of orthogonal matrices. Thus we may apply Lemma~\ref{lem:conv_orthogonal}: this implies $\lambda_i(Z_t^\top C^2 Z_t) \leq \lambda_i(C^2)$. If $\beta_i$ and $\gamma_i$ are the eigenvalues of $B$ and $C$ respectively, we have
\[
\begin{split}
    \norm{C Z^\top B}_F^2 &= \tr{B Z^\top C^2 Z B} = \tr{B^2 Z^\top C^2 Z} \leq \sum_{i=1}^n \beta_i^2 \gamma_i^2
\end{split}
\]
by von Neumann's inequality (Lemma~\ref{lem:von-neumann}). Thus
\[
\norm{\nabla f(Z_t)}_F \leq 2 \sqrt{ \sum_{i=1}^n \beta_i^2 \gamma_i^2 }.
\]
Choosing $L_f = 2 \sqrt{ \sum_{i=1}^n \beta_i^2 \gamma_i^2 }$ in our invocation of Theorem~\ref{lemma:SO_conc}, we obtain 
\[
\norm{f(U) - \mathbb{E}\left[ f(U) \right]}_{\psi_2} \leq \frac{2 \alpha}{\sqrt{n}} \sqrt{ \sum_{i=1}^n \beta_i^2 \gamma_i^2}
\]
for some universal constant $\alpha$. Lemma~\ref{lem:concentration} then yields
\[
\begin{split}
    &\Pr_{U \sim SO(n)} \left( \left| f(U) - \mathbb{E} \left[ f(U) \right] \right| \geq t \right)\leq \\&2 \exp \left(- \frac{n t^2}{ 24 \alpha^2 } \left( \sum_{i=1}^n \beta_i^2\gamma_i^2 \right)^{-1} \right). 
\end{split}
\]
With $t = \eps \sqrt{\sum_{i=1}^n \beta_i^2\gamma_i^2 }$, this implies
\[
\begin{split}
&\Pr_{U \sim SO(n)} \bigg( \left| \frac{1}{m} \tr{ A U C U^\top A^\top} - \frac{1}{mn} \tr{C} \tr{A A^\top} \right|\\
&\qquad\qquad\qquad \geq \frac{\eps}{m} \sqrt{\sum_{i=1}^n \beta_i^2\gamma_i^2 } \bigg) \leq 2 e^{-\frac{ n \eps^2}{24 \alpha^2}}. 
\end{split}
\]
The proof follows by choosing $\kappa_2 = 24 \alpha^2$. We now show the same bound holds if $U \sim O(n)$. First note that for each $X\in O(n)\backslash SO(n)$ the transformation $U\mapsto XU$ takes the connected component $O(n)\backslash SO(n)$ of $O(n)$ to $SO(n)$. If further $X$ satisfies that $f(U)=f(XU)$ for all $U\in O(n)$, that would imply that $f(U)$ has the same distribution whether $U$ is sampled from $SO(n)$ or from $O(n)$, proving the theorem. To achieve this, choose any $X$ that has the same eigenbasis as $C$, with eigenvalues in $\{1,-1\}$ and with determinant $-1$.
%We now show the same bound holds if $U \sim O(n)$. If $n$ is odd, we observe $U \in SO(n)$ if and only if $-U \notin SO(n)$ (as $\text{det}(-U) = -\text{det}(U)$). Further, as $f(U) = f(-U)$ we conclude that $f$ has the same distribution regardless if $U$ is sampled from $O(n)$ or $SO(n)$: 
%\youngsuk{A big jump for people like me. uniformly sampled from $SO(n)$ <-> orthogonal matrix $U$ sampled from \textbf{the} Harr measure on $O(n)$? it must be from random matrix theory... }
%the claim follows. If $n$ is even, 
%\youngsuk{The sentence is incomplete}
\end{proof}

\begin{remark} \rm
We observe that the original proof of Theorem~\ref{thm:janzing} is faulty: using the proof method in \citep{janzing2009telling} we can only ensure the failure probability is at most $n \exp \left( \kappa_1 (n-1) \eps^2 \right)$. The original proof given there uses L\'evy's concentration lemma on each column of the matrix $U$ drawn from $O(n)$. 
However, while the marginal distribution of each column is a uniform vector on $\mathbb{S}^{n-1}$ (i.e. the $n$-dimensional sphere), 
the columns are not independent. This induces a loss of a factor of $n$ from the union bound over the $n$ columns of $U$. 
\end{remark}

\iffalse %i don't think this is needed in the main body

\begin{remark}
The fact that this inequality improves as a function of $m$ should not be surprising. We observe that we may assume $B$ and $C$ are diagonal matrices without loss of generality (as $X$ is drawn from a rotationally invariant distribution). We then have
\[
\tr{ B X C X^\top} = b^\top \left( X \circ X \right) c,
\]
where $\circ$ denotes the Hadamard (entrywise) product. Make the simplifying assumption that $b$ and $c$ are $0-1$ vectors with $m/2$ and $n/2$ nonzero entries respectively (the factor of two avoids shenanigans that occur with $C$ being the identity matrix)  Now, each entry of $X$ is distributed as a Gaussian random variable with mean $0$ and variance $\frac{1}{n}$. Therefore, the square of each entry of $X$ is distributed as $\frac{1}{n}$ times a standard chisquare random variable with one degree of freedom, and so heuristically $b^\top \left( X \circ X \right) c$ is distributed as $\frac{1}{n}$ times $\chi_{mn/4}$. This has mean $m/4$ (which matches the actual expected value of $f$) and variance roughly $\frac{mn}{4 n^2} = \frac{m}{4n}$. However, our concentration inequality is phrased in terms of $\frac{1}{m} f(X)$: this heuristic says its variance is $\frac{1}{4mn}$. This factor of $n$ appears in the exponentially small failure probability, but the above calculation says that it should actually be a factor of $mn$.
\end{remark}

\fi

\subsection{Extension to nonlinear trace functionals}

Next, we describe an extension to the (Linear) Trace Method to \emph{nonlinear} functions of the input covariance matrices. For any $p \geq 1$, we define a variant of the Delta estimator as 
\[
\Delta^{(p)}_{X \rightarrow Y} :=   \log \frac{\tau_m \left( \left( A  C_{XX} A^\top \right)^p \right) }{\mathbb{E}_{U \sim O(n)} \tau_m  \left( \left( A  U C_{XX} U^\top A^\top \right)^p \right)} .
\]
For the case of $p = 1$, this estimator is equivalent to the standard Delta estimator \eqref{eqn:deltaEstimator}. We show that under similar generative assumptions to those in Theorem~\ref{thm:ours}, the quantities in the $\Delta^{(p)}$ estimator concentrate around their mean. %Due to space limit, the complete proof is omitted here and can be found in Appendix~\ref{sec:thmnl}.

\begin{theorem}
\label{thm:ours2}
Under the same assumptions as in Theorem~\ref{thm:janzing}, denote $\beta_i, \gamma_i$ as the eigenvalues of matrices $A^\top A$ and $C$ respectively. For any $p \geq 1$, we have
\begin{align*}
& \bigg| \tau_m \left( \left( A  U C U^\top A^\top \right)^p \right)\\
& - \mathbb{E}_{U \sim O(n)} \left[ \tau_m \left( \left( A U C U^\top A^\top \right)^p \right) \right] \bigg| \leq \frac{p \eps}{m} \sqrt{\sum_{i=1}^n \beta_i^{2p} \gamma_i^{2p} } 
\end{align*}
with probability at least $1 - 2\exp(-\kappa_2  n \eps^2)$ for some universal constant $\kappa_2$.
\end{theorem}

\begin{proof}

% and $h(X) = g(X)^{1/p}$

%We first analyze $g(X)$:  to do this we need to bound the matrix derivative of $g(X)$.

We show how to extend the analysis of Theorem~\ref{thm:ours} to the function $g(U) = \tr{(BUCU^\top)^p}$, for any $p \geq 1$. We again seek to apply the concentration of Lipschitz functions on the sphere. By direct calculation we have 
\[
\nabla g(U) = 2p (CU^\top B U)^{p-1} C U^\top B
\]

Now, define $D = U C U^\top$, we have
\begin{align*}
&\norm{\nabla g(U)}_F^2 \\
&\ = 4 p^2 \tr{(C U^\top B U)^{p-1} C U^\top B^2 U C (U^\top B U C)^{p-1}} \\
&\ = 4p^2 \tr{C U^\top B (DB)^{p-1}  (BD)^{p-1} B U C } \\ 
&\ = 4p^2 \tr{C U^\top B (DB)^{p-1}  (BD)^{p-1} B U C U^\top U} \\
&\ = 4p^2 \tr{(DB)^p (BD)^p}. 
\end{align*}

We now recall the Lieb-Thirring inequality: for any symmetric matrices $A, B$ and any $p \geq 1$, $\tr{(AB)^p} \leq \tr{A^p B^p}$. Employing this and the matrix Cauchy-Schwarz inequality yields
\iffalse
\hilaf{Let's not recall the first one; we have already used it in the past without making a lemma for it. When we use it we will write in parentheses that it's the Matrix Cauchy Schwarz.}
\begin{lemma}[Matrix Cauchy Schwarz]
For any two matrices $X,Y$, we have 
\[
\tr{XY}^2 \leq \tr{X X^\top} \tr{Y Y^\top}.  
\]
\end{lemma}
\youngsuk{weird to have lemma in the middle of proof. May better say 'where we apply Lieb-Thirring Inequality inequality $\tr{(AB)^p} \leq \tr{A^p B^p}$' below}
\youngsuk{If Arun is okay, I can edit by myself.}
\fi
\[
\norm{\nabla g(U)}_F^2 \leq 4p^2 \tr{ (BD)^{2p}} \leq 4p^2 \tr{ B^{2p} D^{2p}}. 
\]
Now for any $U = (1-t) Y + t Z$ with $Y,Z \in O(n)$ Lemma~\ref{lem:conv_orthogonal} implies
\[
\lambda_i (D) = \lambda_i \left( U C U^\top \right) \leq \lambda_i (C). 
\]
Applying von Neumann's inequality as before yields 
\[
\norm{\nabla g(U)}_F \leq 2p \sqrt{ \sum_{i=1}^n \beta_i^{2p} \gamma_i^{2p}}
\]
where $\beta_i$, $\gamma_i$ are $B$, $C$'s eigenvalues respectively. Repeating the analysis of Theorem~\ref{thm:ours} implies the claim. 

\end{proof}
%We now turn to $h(X)$. Observe that 
%\[
%\nabla h(X) = \frac{1}{p g(X)^{1-1/p}} \nabla g(X).  
%\]
%Applying our calculations for $g(X)$, we obtain
%\[
%\norm{\nabla h(X)}_F^2 \leq \frac{4 \tr{(BD)^{2p}}}{ \tr{ (BD)^{p}}^{2-2/p}}.
%\]
%This term should be upper bounded by $4 \norm{D^{1/2} B D^{1/2}}^2$, where the norm is the standard matrix spectral norm. 
%\hilaf{TODO: Extract in a corollary, and add a remark about the implication of this, how it can be used.}

%\hilaf{Trying to tie this to the previous section, you're suggesting using $h$ instead of $f$ in Theorem \ref{thm:janzing}. In that case, what is $\mathbb{E}h(X)?$}

With this concentration result in hand, we will describe some potential use cases for such generalized trace estimators. We observe that
\[
\begin{split}
    \frac{p \epsilon}{m} \sqrt{ \sum_{i=1}^n \beta_i^{2p} \gamma_i^{2p}} &\leq \frac{p \epsilon}{m} \sqrt{\mathsf{rank}(A A^\top) \beta_1^{2p} \gamma_1^{2p}} \\&\leq \frac{p \epsilon}{\sqrt{m}} \norm{A A^\top}^p \norm{C}^p.
\end{split}
\]
After taking a logarithm, the resulting bounds on the nonlinear Delta estimator in the causal direction are at most a factor of $p$ off of the corresponding bounds from \cref{thm:ours}. However, in cases where either $A A^\top$ or $C$'s eigenvalues are non-uniform, the above bound is loose. More specifically if $C$'s spectrum is dominated by a small number of eigendirections the resulting confidence bounds on the Delta estimator can improve by roughly a factor of $\sqrt{m}$: this is often the case in practice and warrants future experimental investigation.

% \youngsuk{In a slightly different context, can we say how much the design of eigenvalue distribution (or largest one or its energy) affect this upperbound? what do we recommend to choose $p$ accordingly?. }

% \youngsuk{Putting differently, can we say any benefits of generalized delta estimator $p>1$ compared with $p=1$?}

\section{Ridge-regularized Delta Estimator}
\label{sec:rr}
In this section, we describe a novel empirical estimator for $\Delta_{X \rightarrow Y}$, based on the standard ridge estimator of the coefficient matrix $A$. Our estimator operates under the ``big-data'' regime: we assume the causal variable $X$ lies in $\R^n$ and that we have $cn$ samples $(X_i, Y_i)$ from a Linear causal model (\Cref{def:model}) for fixed $c \geq 0$. We define the ridge-regularized Delta estimator as
\begin{equation}
\label{eqn:ridge}
\widetilde{\Delta}_{X \rightarrow Y}^{\lambda} := \log \frac{\tau_m ( \hat{A}_{\lambda} C_{XX} \hat{A}_{\lambda}^\top )} {\tau_m(\hat{A}_{\lambda} \hat{A}_{\lambda}^\top) \tau_n (C_{XX}) } 
\end{equation}
where $C_{XX}$ is the empirical estimate for the covariance of $X$ and the estimate of the structural matrix $A$ is obtained by 
\[
\hat{A}_{\lambda} = \arg\min_{A} \frac{1}{cn}  \sum_{i=1}^{cn} \norm{ A X_i - Y_i }_2^2 + \lambda \norm{A}_F^2
\]
with $\lambda \geq 0$ a chosen hyper-parameter. Algorithm~\ref{algo:ridge} shows the complete procedure to test causality using the proposed estimator. 
% \youngsuk{Mention the high level goal. e.g., we use $\alpha_\lambda$ estimator to .... and plug in the test.}\youngsuk{great to have the whole test procedure, including choosing parameters (or given), estimator, test in algorithmic style, e.g.,  Janzing's algorithm 1 }

\RestyleAlgo{ruled}
\begin{algorithm}[t]
\caption{Ridge-regularized Delta Estimator for Causality Testing}
\label{algo:ridge}
\KwInput{Samples $(X_i, Y_i)$ from a Linear causal model, with $X_i \in \R^n$, $Y_i \in \R^m$, ridge parameter $\lambda \geq 0$, significance level $\xi \geq 0$}
\KwOutput{Prediction for the causal direction $X \rightarrow Y$, $Y \rightarrow X$, or inconclusive}
$C_{XX} \gets $ estimate for covariance of $X$ \;

$C_{YY} \gets $ estimate for covariance of $Y$ \;

$C_{XY} \gets $ estimate for cross-covariance of $X$ and $Y$ \;

$\hat{A}_{XY} = \left( \sum_i X_i X_i^\top + \lambda I\right)^{-1} \sum_i X_i Y_i$,\quad $\hat{A}_{YX} = \left( \sum_i Y_i Y_i^\top + \lambda I\right)^{-1} \sum_i Y_i X_i$ \;

$\widetilde{\Delta}^{\lambda}_{X \rightarrow Y} = \log \left(\dfrac{ \tau_m ( \hat{A}_{XY} C_{XX} \hat{A}_{XY}^\top )}{\tau_m (\hat{A}_{XY} \hat{A}_{XY}^\top ) \tau_n (C_{XX})} \right), $\quad $\widetilde{\Delta}^{\lambda}_{Y \rightarrow X} = \log \left(\dfrac{ \tau_m ( \hat{A}_{YX} C_{YY} \hat{A}_{YX}^\top )}{\tau_m (\hat{A}_{YX} \hat{A}_{YX}^\top ) \tau_n (C_{YY}) }\right) $\;

\uIf{$\widetilde{\Delta}^{\lambda}_{X \rightarrow Y} \geq \widetilde{\Delta}^{\lambda}_{Y \rightarrow X} + \xi $}{
Return ``$X$ is the cause of $Y$''\;
}
\uElseIf{$\widetilde{\Delta}^{\lambda}_{Y \rightarrow X} \geq \widetilde{\Delta}^{\lambda}_{X \rightarrow Y} + \xi $}{
Return ``$Y$ is the cause of $X$''\;
}
\Else{
Return ``Inconclusive.''
}
\end{algorithm}

In the rest of the section, we prove upper and lower bounds on the terms of our empirical Delta estimator relying on ridge regression in the presence of noise and the limit of $n \rightarrow \infty$. To facilitate the practical usage of our estimator, we propose a strategy to choose an appropriate regularization parameter and to relate these bounds to an overall bound on $\widetilde{\Delta}_{X \rightarrow Y}^{\lambda}$.

\subsection{Theoretical Properties of $\widetilde{\Delta}_{X \rightarrow Y}^{\lambda}$}

% In this section, we prove upper and lower bounds on our empirical Delta estimator in the presence of noise and the limit of $n \rightarrow \infty$. 

%\arun{prove upper + lower bounds on the Delta estimator.}

In this section, we provide asymptotic upper and lower bounds on the terms of $\widetilde{\Delta}^{\lambda}$ given (noisy) data in the causal direction, under the additional assumption that the data is zero-mean. Unfortunately, due to the lack of linearity we were unable to give bounds on the asymptotic behavior of $\widetilde{\Delta}^{\lambda}$ itself: we instead prove that $\tau_m (\hat{A}_{\lambda} C_{XX} \hat{A}_{\lambda}^\top )$ and $\tau_m (\hat{A}_{\lambda} \hat{A}_{\lambda}^\top )$ lie close to their corresponding ground-truth values $\tau_m ( A \Sigma_{XX} A^\top)$ and $\tau_m ( A A^\top) $ respectively. 

\begin{lemma}
\label{lem:ridge_main}
In the setup above, under the assumption that the model is Linear causal (Definition \ref{def:model}), and the additional assumptions that the random noise $E$ satisfies $\mathbb{E} \left[ E E^\top \right] \preceq \sigma^2 I$ and $\mathbb{E} \left[ X \right] = 0$, and letting $\lambda$ be chosen uniformly at random between $0$ and $\lambda'$, we have the following:
%Let $(X_i, Y_i)$ be samples from a Gaussian causal model with causal variable $X$ (Definition \ref{def:model}). Assume the random noise in the model $E$ satisfies $\mathbb{E} \left[ E E^\top \right] = \sigma^2 I$, and assume $\mathbb{E} \left[ X \right] = 0$. Let $\lambda$ be chosen uniformly at random between $0$ and $\lambda'$. If $X_i \in \R^n$ and we receive $cn$ samples from the model (for a fixed constant $c \geq 0$), $\hat{A}_{\lambda}$ has 
\[
\begin{split}
&-\frac{\lambda' \norm{A}_F^2}{n} \tr{ (p_n \Lambda_{XX} + \lambda' I)^{-1} } - \mathcal{E}\\
&\qquad\qquad\leq\mathbb{E} \left[ \norm{\hat{A}_{\lambda}}_F^2 - \norm{A}_F^2 \right]\leq \\
&\left(\frac{\sigma^2}{cn}-\frac{\lambda' \norm{A}_F^2}{n}\right) \tr{ (p_n \Lambda_{XX} + \lambda' I)^{-1} } + \mathcal{E}
\end{split}
\]
where $\mathcal{E} \rightarrow 0$ as $n \rightarrow \infty$, and where $p_n$ is the unique solution to
\[
1 -p_n = \frac{p_n}{cn} \tr{ \Sigma_{XX} \left( p_n \Sigma_{XX}  + \lambda' I \right)^{-1} }.
\]
\end{lemma}

% \youngsuk{Please check: $n$ seems to be the number of samples, which conflicts with dimension of $X$.} \arun{we have $cn$ samples in $n$ dimensions}
% \youngsuk{minor: may want to use different notation for $p_n$ as $p$ is used as exponent in estimator? When describing sample bound for $\hat \Delta^{(p)}$, the confusion will come up.}
\begin{proof}
Given our samples $(X_i, Y_i)$, 
form the matrices $C_{XX} = \frac{1}{cn} \sum_{i=1}^{cn} X_i X_i^\top$ and $Z = \frac{1}{cn} \sum_{i=1}^{cn} X_i Y_i^\top$. With these, we have $\hat{A}_{\lambda}^\top = (C_{XX}+ \lambda I)^{-1} Z$. By the definition of the linear model, we observe
\[
\begin{split}
    Z &= \frac{1}{cn} \left( \sum_{i=1}^{cn} X_i X_i^\top A^\top + X_i E_i^\top \right)\\&=   C_{XX} A^\top + \frac{1}{cn} \sum_i X_i E_i^\top. 
\end{split}
\]
Defining $\Gamma = \frac{1}{cn} \sum_i X_i E_i^\top$, we have
\[
\begin{split}
&\norm{\hat{A}_{\lambda}}_F^2  = \tr{A C_{XX} (C_{XX}+\lambda I)^{-2} C_{XX} A^\top}\\
& \quad + 2 \tr{\Gamma^\top (C_{XX}+ \lambda I)^{-2} C_{XX} A^\top}\\
& \quad + \tr{ \Gamma^\top  (C_{XX}+\lambda I)^{-2} \Gamma }.
\end{split}
\]
Since $E$ is zero-mean, we have $\mathbb{E}\left[ \Gamma \right] = 0$. Further, since $\mathbb{E} \left[ E_i E_i^\top \right] \preceq \sigma^2 I$, 
\[
\begin{split}
\mathbb{E}_{E_i} \left[ \norm{\hat{A}_{\lambda}}_F^2 \right] &= \tr{ A C_{XX} (C_{XX}+\lambda I)^{-2} C_{XX} A^\top } + \epsilon_{\lambda},
\end{split}
\]
where $\epsilon_{\lambda}$ is between $0$ and $\frac{\sigma^2}{cn} \tr{ (C_{XX}+\lambda I)^{-2} C_{XX}}$. 
Taking the expectations over the random choice of $\lambda$, 
\[
\begin{split}
&\mathbb{E}_{\lambda} \left[ C_{XX} (C_{XX}+\lambda I)^{-2} C_{XX} \right]\\&\quad= \frac{1}{\lambda'} \int_0^{\lambda'} C_{XX} (C_{XX}+\lambda I)^{-2} C_{XX} \mathrm{d} \lambda \\
&\quad = C_{XX} (C_{XX}+\lambda' I)^{-1} = I - \lambda' (C_{XX}+\lambda' I)^{-1}.
\end{split}
\]
and
\[
\begin{split}
&\mathbb{E}_{\lambda} \left[ \tr{ (C_{XX}+\lambda I)^{-2} C_{XX}} \right] \\&\quad = \frac{1}{\lambda'} \int_0^{\lambda'} \tr{ (C_{XX}+\lambda I)^{-2} C_{XX}} \mathrm{d} \lambda  = \tr{ (C_{XX}+\lambda' I)^{-1} }.
\end{split}
\]
Thus 
\[
\begin{split}
&\mathbb{E}_{E_i, \lambda} \left[ \norm{\hat{A}_{\lambda}}_F^2 \right] - \norm{A}_F^2 = - \lambda' \tr{A^\top (C_{XX}+\lambda' I)^{-1} A} + \epsilon
\end{split}
\]
for $\epsilon \in [0, \frac{\sigma^2}{cn} \tr{(C_{XX}+\lambda' I)^{-1}}]$. Now observe that $C_{XX}$ is the empirical covariance of $X$. By the generalized Marchenko-Pastur law~\citep{liu2019ridge}, we have $
(C_{XX} + \lambda' I)^{-1} \rightarrow (p_n \Sigma_{XX} + \lambda' I)^{-1}
$
in distribution as $n \rightarrow \infty$, where $p_n$ is the unique solution to 
\[
1 -p_n = \frac{p_n}{cn} \tr{ \Sigma_{XX} \left( p_n \Sigma_{XX}  + \lambda' I \right)^{-1} }.
\]
Now, recall that $\Sigma_{XX} = U \Lambda_{XX} U^\top$ for $U \sim O(n)$ by the generative assumption. Since $p_n$ is unchanged by orthogonal rotations,
\[
\mathbb{E} \left[ (p_n \Sigma_{XX} + \lambda' I)^{-1} \right] = \frac{1}{n} \tr{ (p_n \Lambda_{XX} + \lambda' I)^{-1} } I.
\]
Applying this, we arrive at the desired bounds
\[
\begin{split}
&\mathbb{E} \left[ \norm{\hat{A}_{\lambda}}_F^2 - \norm{A}_F^2 \right] \geq -\frac{\lambda' \norm{A}_F^2}{n} \tr{ (p_n \Lambda_{XX} + \lambda' I)^{-1} } - \mathcal{E}
\end{split}
\]
and
\[
\begin{split}
\mathbb{E} \left[ \norm{\hat{A}_{\lambda}}_F^2 - \norm{A}_F^2 \right] &\leq -\frac{\lambda' \norm{A}_F^2}{n} \tr{ (p_n \Lambda_{XX} + \lambda' I)^{-1} }\\
&+ \frac{\sigma^2}{cn} \tr{ (p_n \Lambda_{XX} + \lambda' I)^{-1} } + \mathcal{E},
\end{split}
\]
where the expectation is taken over the samples $(X_i, Y_i)$, the random choice of $\lambda$, and the randomness used to generate the covairance matrix $\Sigma$.
%\hilaf{Fix the above styling.}
\end{proof}

\begin{remark}
We note that converting the above bound in the limit of large $n$ to a finite-$n$ bound is difficult due to the reliance on the Marchenko-Pastur distributional law~\citep{pastur1967distribution, anderson2010introduction, edelman2005random}. However, in practice the deviation between the true distribution of $\sum_{i} X_i X_i^\top$ and the Marchenko-Pastur law is negligible for moderately large $n$: we consequently employ this assumption in our analysis. 
\end{remark}

By adapting the proof technique, we can additionally quantify the bias of the numerator in our empirical Delta estimator. %In the interest of space, we leave the sketch of the proof in the Appendix.
\begin{lemma}
\label{lemma:num}
Under the same assumptions as Lemma~\ref{lem:ridge_main}, we have
\[
\begin{split}
&- \lambda' \norm{A}_F^2+ \lambda'^2 \tr{A (p_n \Lambda_{XX} + \lambda' I)^{-1} A^\top} - \mathcal{E}\\
&\qquad \leq \mathbb{E} \left[ \tr{\hat{A}_{\lambda} C_{XX} \hat{A}_{\lambda}} - \tr{ A \Sigma_{XX} A^\top } \right] \leq \\&- \lambda' \norm{A}_F^2+ \lambda'^2 \tr{A (p_n \Lambda_{XX} + \lambda' I)^{-1} A^\top} \\
&\quad\qquad\quad\ \ +\frac{\sigma^2}{cn} ( n - \lambda' \tr{( p_n \Lambda_{XX} + \lambda' I)^{-1}} + \mathcal{E}
\end{split}
\]
for $E \rightarrow 0$ as $n \rightarrow \infty$, where $p_n$ is defined as in Lemma~\ref{lem:ridge_main}.
\end{lemma}

\begin{figure*}
    \begin{minipage}{0.33\textwidth}
    \centering
    \includegraphics[width=\textwidth]{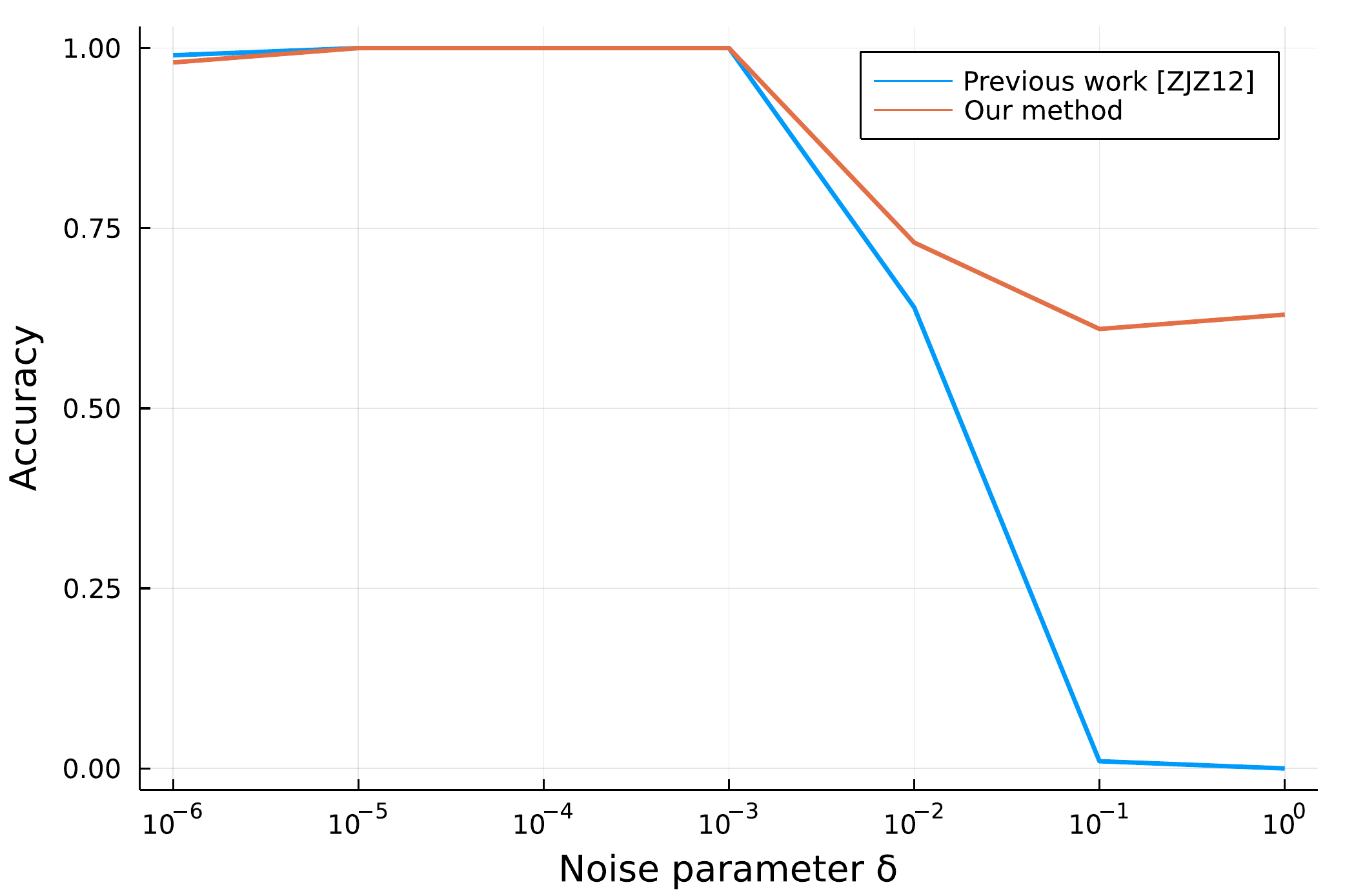}
    % \label{fig:noise}
    \end{minipage}
    \begin{minipage}{0.33\textwidth}
    \centering
    \includegraphics[width=\textwidth]{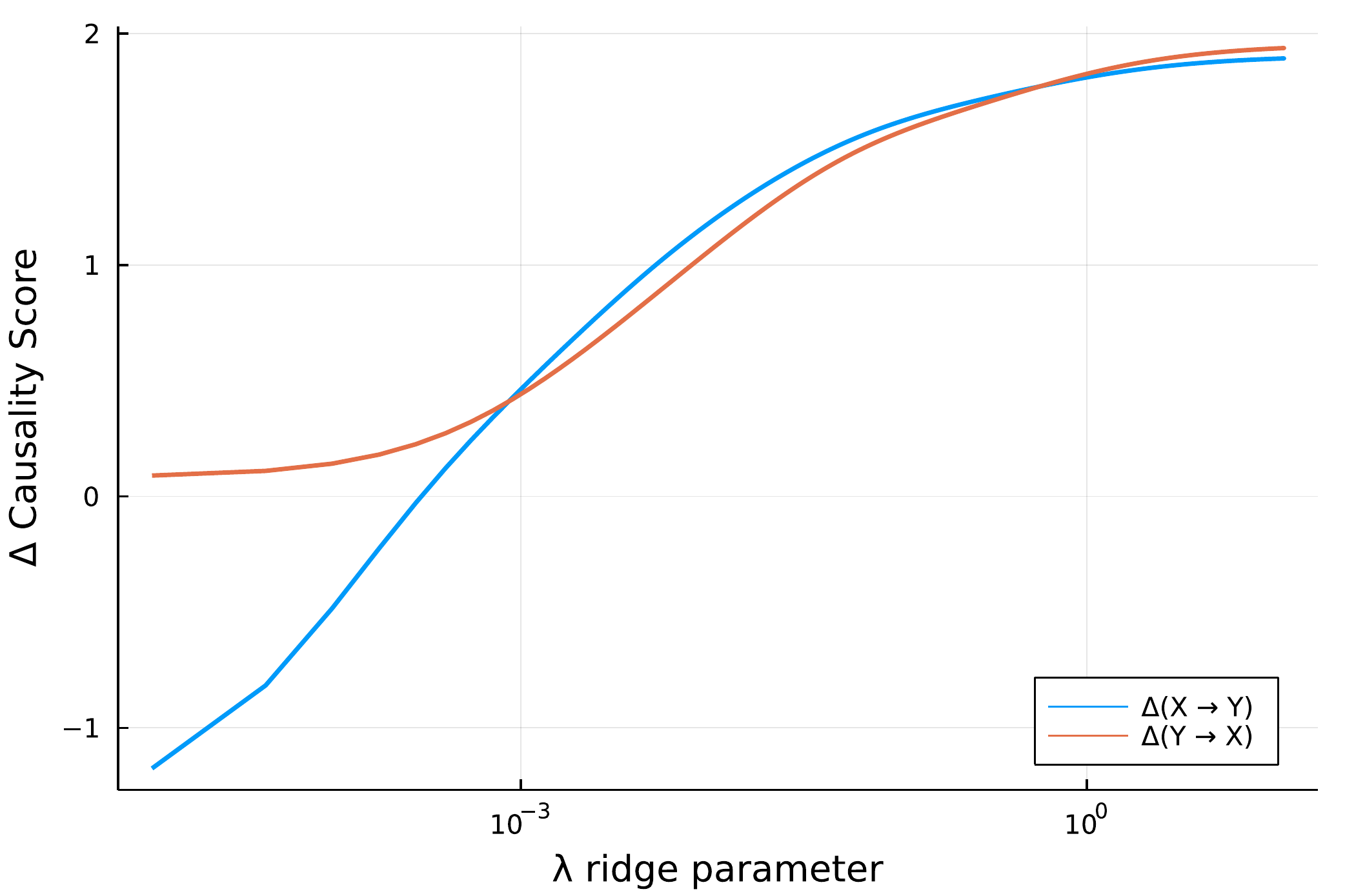}
    \end{minipage}
    \begin{minipage}{0.33\textwidth}
    \centering
    \includegraphics[width=\textwidth]{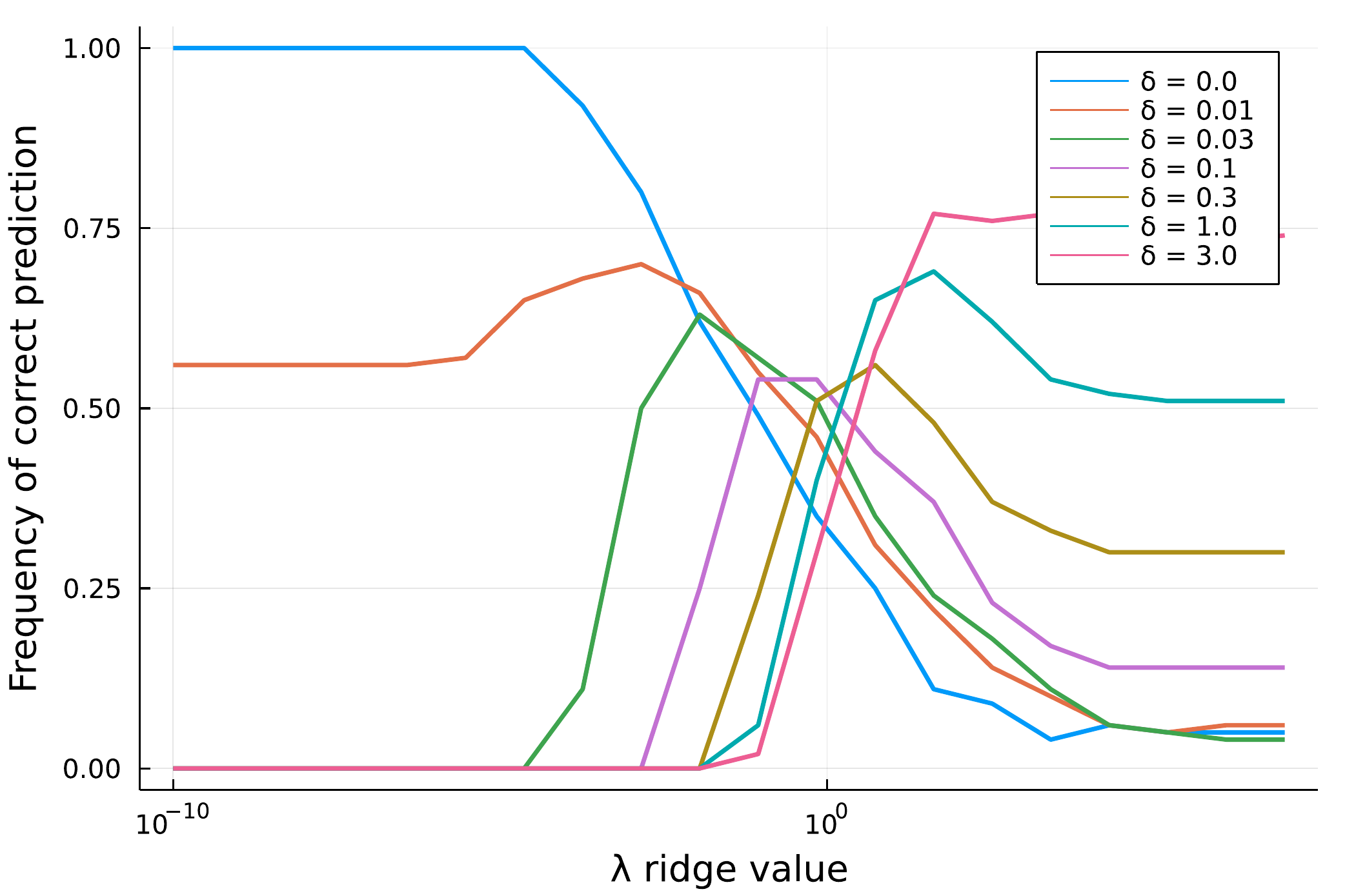}
    \end{minipage}
    \caption{Left: The comparison of the Delta estimator (blue curve) and the proposed ridge estimator (red curve). 
    The $x$-axis denotes the additive noise, and the $y$-axis is the accuracy, i.e., the proportion of the correct conclusion ($X$ causes $Y$) is drawn over 100 trials. Middle: Comparison of the Ridge-regularized Delta performance for when $\delta = 0.03$. Note that for $\lambda \in [10^{-3}, 10^{-1}]$ our algorithm makes the correct causal decision. Right: The accuracy of the proposed estimator under different noise level $\delta$. The $x$-axis denotes the regularization parameter $\lambda$, and the $y$-axis is the accuracy.}
    \label{fig:noise}
\end{figure*}

\begin{proof}[(Sketch)]
By repeating the proof from above, we have 
\[
\begin{split}
&\mathbb{E}_{E_i} \left[ \tr{ \hat{A}_{\lambda} C_{XX} \hat{A}_{\lambda}} \right] \\
&= \tr{ A C_{XX}^{1.5} (C_{XX} + \lambda I)^{-2} C_{XX}^{1.5} A} \\
&+ \frac{\sigma^2}{cn} \tr{C_{XX}^2 (C_{XX} + \lambda I)^{-1}}.
\end{split}
\]
Analogously to before, we have 
\[
\begin{split}
&\mathbb{E}_{\lambda} \left[ C_{XX}^{1.5} (C_{XX} + \lambda I)^{-2} C_{XX}^{1.5} \right] \\
&\qquad= \frac{1}{\lambda'} \int_0^{\lambda'} C_{XX}^{1.5} (C_{XX} + \lambda I)^{-2} C_{XX}^{1.5} \mathrm{d}\lambda \\
&\qquad = C_{XX}(C_{XX} + \lambda' I)^{-1} C_{XX} \\
&\qquad = C_{XX} - \lambda' I +  \lambda'^2 (C_{XX} +  \lambda' I)^{-1}.
\end{split}
\]
Thus 
\[
\begin{split}
&\mathbb{E}_{E_i,  \lambda} \left[ \tr{\hat{A}_{\lambda} C_{XX} \hat{A}_{\lambda}^\top  } \right] = \tr{ A C_{XX} A^\top} - \lambda' \tr{ A A^\top } \\
&\qquad\qquad + \lambda'^2 \tr{ A (C_{XX} + \lambda' I)^{-1} A^\top} + \epsilon,
\end{split}
\]
where $\epsilon \in [0, \frac{\sigma^2}{c} - \frac{\lambda' \sigma^2}{cn} \tr{ (C_{XX} + \lambda' I)^{-1} }]$
Employing the Marchenko-Pastur law (\cite{liu2019ridge}) and noting $\mathbb{E}_{X_i} [C_{XX}] = \Sigma_{XX}$,  we obtain the claimed bounds.
\end{proof}

% \begin{remark}
% Assuming we have an upper bound on the variance of the error variable $E,$ we can show that a desirable choice of the regularization parameter is
% $\lambda' = \frac{\hat{\sigma}^2}{2 c \norm{A}_F^2}.$
% Due to space limit, we provide the analysis in Appendix~\ref{sec:lambda}. 

We now describe a strategy to choose the value of $\lambda$ in our ridge estimator. Assume we had an upper bound $\hat{\sigma}^2$ on the variance of the error variable $E$. The value $p_n$ satisfies
\[
\begin{split}
1 - p_n &= \frac{1}{cn} \tr{p_n \Sigma_{XX} \left( p_n \Sigma_{XX} + \lambda' I \right)^{-1} } \\
&= \frac{1}{cn} \left(n - \lambda' \tr{\left( p_n \Sigma_{XX} + \lambda' I \right)^{-1} } \right)
\end{split}
\]
Rearranging yields
\[
 \tr{\left( p_n \Sigma + \lambda' I \right)^{-1} } = \frac{c n p_n - cn + n}{\lambda'}.
\]
Substituting in and after simplification, we have
\[
\begin{split}
\mathbb{E} \left[ \norm{\hat{A}_{\lambda}}_F^2 - \norm{A}_F^2 \right]  &= \left( - \frac{\lambda' \norm{A}_F^2}{n} + \frac{\sigma^2}{cn} \right) \frac{c n p_n - cn + n}{\lambda'}\\
&= - \norm{A}_F^2 (c p_n - c + 1) + \frac{\sigma^2}{c \lambda'} (c p_n - c + 1).
\end{split}
\]
Choosing 
$
    \lambda' = \frac{\hat{\sigma}^2}{2 c \norm{A}_F^2}, 
$
we have, given that $\sigma^2 \in [0, \hat{\sigma}^2]$ 
\[
\left| \mathbb{E} \left[ \norm{\hat{A}_{\lambda}}_F^2 - \norm{A}_F^2 \right] \right| \leq \norm{A}_F^2 ( c p_n - c +1 ). 
\]
We remark that $p_n$ approaches $1$ as $c$ grows larger: the right hand side is always positive. We would like to call out the practical limitation of setting the regularization parameter to $\lambda'$. The optimal value of lambda implied by our proof is the solution to an implicit equation depending on the eigenvalues of the underlying covariance. Therefore, a finite sample estimation is required to make it practical, and we plan to investigate this in future work.

\section{Experiments}
\label{sec:experiments}
To demonstrate the utility of our ridge estimator, we conduct experiments comparing it against \cite{zscheischler2011testing} on synthetic data. The primary questions that we seek answer to are: 1) In the presence of additive noise, can we recover the correct causal relationship? 2) How does the regularization parameter $\lambda$ affect the effectiveness of the proposed estimator under different level of noise?
% \youngsuk{Specify that we choose $p=1$ unless we are going to do additional experiments to see the effect of varying $p$.}

To this end, we generate the synthetic dataset as follows. First, we sample $X \sim N(0, U \Sigma_X U^\top)$ with $U$ Haar-distributed and the singular values $\Sigma_X$ following a power-law distribution.
%\youngsuk{Or do we want to study over different power law parameters?  I don't think so... then better discuss at the end of section or conclusion?}. 
The structural matrix $A$ has i.i.d. Gaussian entries. We then took $Y = AX + \delta E$, for $E \sim N(0,I)$ and noise parameter $\delta$: clearly, the data generated in this way has a causal relationship from $X$ to $Y$. In the experiments, we generated $T=100$ samples from $n=40$-dimensional Gaussians and compare the estimator \eqref{eqn:naive_delta_est} of \cite{zscheischler2011testing} with the proposed estimator \eqref{eqn:ridge}. Following \cite{zscheischler2011testing}, we declared the causal direction to be the one with a larger $\Delta$ score with a tolerance margin of $\xi=0.05$, i.e., $\Delta_{X\rightarrow Y} \geq \Delta_{Y\rightarrow X} + 0.05.$ As shown in Figure~\ref{fig:noise} Left, Delta estimator consistently fails to correctly identify the causal direction in the presence of even modest amounts of noise. On the other hand, our estimator with an appropriate choice of regularization parameter $\lambda$ can recover the correct causal relationship between $X$ and $Y$, even in the presence of moderate amount of noise.

% \begin{figure}
%     \centering
%     \includegraphics[scale=0.33]{icml2022/p3.png}
%     \caption{Performance of \cite{janzing2009telling} in the presence of noise.}
%     \label{fig:my_label}
% \end{figure}

Next we quantify how the regularization parameter affect the conclusion of the causal direction. To this aim, we computed our estimator with $\delta = 0.03$ and various values of ridge parameter $\lambda$, and check the value of the ridge estimators for $X\rightarrow Y$ and $Y\rightarrow X.$ For small values of $\lambda$, the noise latent in our causal model hides the fact that $X$ causes $Y$. For large values of ridge parameter, the ridge-estimated coefficient matrix fails to be a good approximation of the truth: we may reach the wrong conclusion. Within a `Goldilocks' region however, the value of $\lambda$ is sufficient to compensate for the introduced noise but does not obscure the coefficient matrix. We remark that for the noise level chosen ($\delta = 0.03$) the estimator of \cite{janzing2009telling} strongly and incorrectly indicates that $Y$ is the cause of $X$.
As we see in Figure~\ref{fig:noise} Center, our algorithm satisfies $\Delta_{X \rightarrow Y} \geq \Delta_{Y \rightarrow X} + 0.05$ for a significant regime of $\lambda$: it thus correctly identifies the causal relationship. Encouraged by this result, we generated $100$ different $C_{XX}$ Linear causal models  $Y = AX + E$ for various values of $\delta$ and, for a range of $\lambda$ regularization values. We give the results below in Figure~\ref{fig:noise} Right.
% \vspace{-0.3cm}
% \begin{figure}[H]
%     \centering
%     \includegraphics[scale=0.33]{p2.pdf}
%     \caption{The accuracy of the proposed estimator under different noise level $\delta$. The $x$-axis denotes the regularization parameter $\lambda$, and the $y$-axis is the accuracy.}
%     \label{fig:lambda}
% \end{figure}
% \vspace{-0.5cm}
%As it can be seen, the optimal lambda from the experiments (i.e., the parameter that achieves the highest accuracy) matches closely to the guideline from the theoretical estimate in Section~\ref{sec:lambda}. 
The results achieved in these experiments confirmed our conclusions from the previous experiment. We observed a sharp dip after a mild amount of added noise before the performance of our estimator improves with higher noise. We suspect that this is caused by moderate amounts of noise `hiding' critical small eigenvectors of the covariance matrix: once the magnitude of the noise grows beyond this critical region the standard Gaussian noise fails to hide this information and our method's accuracy recovers.

\section{Conclusion}
In this paper, we revisit the linear trace method for testing the causal relationship between two high-dimensional random variables. We significantly improve the state-of-the-art bounds for the linear causality estimator proposed by~\cite{janzing2009telling} (at least by a factor of $O(1/\sqrt{m})$), and extend the results to nonlinear trace functionals with sharper confidence bounds under certain distributional assumptions. As another main contribution, we propose a novel ridge-regression estimator that enjoys provable guarantees for individual terms under noisy regime, a first result of this kind, with numerical simulation showing its promise in high-dimension low sample size setting. However, in this work, we are unable to achieve the finite sample bound due to the lack of a refined tailed analysis of the empirical covariance matrix, which we target for future work.

\newpage
\bibliographystyle{icml2023}
\bibliography{causal}

\begin{thebibliography}{11}
\providecommand{\natexlab}[1]{#1}
\providecommand{\url}[1]{\texttt{#1}}
\expandafter\ifx\csname urlstyle\endcsname\relax
  \providecommand{\doi}[1]{doi: #1}\else
  \providecommand{\doi}{doi: \begingroup \urlstyle{rm}\Url}\fi

\bibitem[Anderson et~al.(2010)Anderson, Guionnet, and
  Zeitouni]{anderson2010introduction}
Anderson, G.~W., Guionnet, A., and Zeitouni, O.
\newblock \emph{An introduction to random matrices}.
\newblock Number 118. Cambridge university press, 2010.

\bibitem[Edelman \& Rao(2005)Edelman and Rao]{edelman2005random}
Edelman, A. and Rao, N.~R.
\newblock Random matrix theory.
\newblock \emph{Acta numerica}, 14:\penalty0 233--297, 2005.

\bibitem[Jambulapati et~al.(2020)Jambulapati, Li, and Tian]{Jambulapati0T20}
Jambulapati, A., Li, J., and Tian, K.
\newblock Robust sub-gaussian principal component analysis and
  width-independent schatten packing.
\newblock In Larochelle, H., Ranzato, M., Hadsell, R., Balcan, M., and Lin, H.
  (eds.), \emph{Advances in Neural Information Processing Systems 33: Annual
  Conference on Neural Information Processing Systems 2020, NeurIPS 2020,
  December 6-12, 2020, virtual}, 2020.

\bibitem[Janzing et~al.(2009)Janzing, Hoyer, and
  Sch{\"o}lkopf]{janzing2009telling}
Janzing, D., Hoyer, P.~O., and Sch{\"o}lkopf, B.
\newblock Telling cause from effect based on high-dimensional observations.
\newblock \emph{arXiv preprint arXiv:0909.4386}, 2009.

\bibitem[Liu \& Dobriban(2019)Liu and Dobriban]{liu2019ridge}
Liu, S. and Dobriban, E.
\newblock Ridge regression: Structure, cross-validation, and sketching.
\newblock \emph{arXiv preprint arXiv:1910.02373}, 2019.

\bibitem[Pastur \& Martchenko(1967)Pastur and
  Martchenko]{pastur1967distribution}
Pastur, L. and Martchenko, V.
\newblock The distribution of eigenvalues in certain sets of random matrices.
\newblock \emph{Math. USSR-Sbornik}, 1\penalty0 (4):\penalty0 457--483, 1967.

\bibitem[Pearl(2009)]{pearl2009causality}
Pearl, J.
\newblock \emph{Causality}.
\newblock Cambridge university press, 2009.

\bibitem[Peters et~al.(2017)Peters, Janzing, and
  Sch{\"o}lkopf]{peters2017elements}
Peters, J., Janzing, D., and Sch{\"o}lkopf, B.
\newblock \emph{Elements of causal inference: foundations and learning
  algorithms}.
\newblock The MIT Press, 2017.

\bibitem[Sch{\"o}lkopf(2019)]{scholkopf2019causality}
Sch{\"o}lkopf, B.
\newblock Causality for machine learning.
\newblock \emph{arXiv preprint arXiv:1911.10500}, 2019.

\bibitem[Vershynin(2018)]{HDP}
Vershynin, R.
\newblock High-dimensional probability: An introduction with applications in
  data science.
\newblock 47, 2018.

\bibitem[Zscheischler et~al.(2011)Zscheischler, Janzing, and
  Zhang]{zscheischler2011testing}
Zscheischler, J., Janzing, D., and Zhang, K.
\newblock Testing whether linear equations are causal: a free probability
  theory approach.
\newblock In \emph{Proceedings of the Twenty-Seventh Conference on Uncertainty
  in Artificial Intelligence}, pp.\  839--848, 2011.

\end{thebibliography}

%%%%%%%%%%%%%%%%%%%%%%%%%%%%%%%%%%%%%%%%%%%%%%%%%%%%%%%%%%%%%%%%%%%%%%%%%%%%%%%
%%%%%%%%%%%%%%%%%%%%%%%%%%%%%%%%%%%%%%%%%%%%%%%%%%%%%%%%%%%%%%%%%%%%%%%%%%%%%%%
% APPENDIX
%%%%%%%%%%%%%%%%%%%%%%%%%%%%%%%%%%%%%%%%%%%%%%%%%%%%%%%%%%%%%%%%%%%%%%%%%%%%%%%
%%%%%%%%%%%%%%%%%%%%%%%%%%%%%%%%%%%%%%%%%%%%%%%%%%%%%%%%%%%%%%%%%%%%%%%%%%%%%%%
% \newpage
% \appendix
% \onecolumn

% You can have as much text here as you want. The main body must be at most $8$ pages long.
% For the final version, one more page can be added.
% If you want, you can use an appendix like this one, even using the one-column format.
%%%%%%%%%%%%%%%%%%%%%%%%%%%%%%%%%%%%%%%%%%%%%%%%%%%%%%%%%%%%%%%%%%%%%%%%%%%%%%%
%%%%%%%%%%%%%%%%%%%%%%%%%%%%%%%%%%%%%%%%%%%%%%%%%%%%%%%%%%%%%%%%%%%%%%%%%%%%%%%

\end{document}